\documentclass{article}

\usepackage{epsfig}
\usepackage{amssymb}
\usepackage{natbib}

\usepackage[utf8]{inputenc} 
\usepackage[T1]{fontenc}    
\usepackage{hyperref}       
\usepackage{url}            
\usepackage{booktabs}       
\usepackage{amsfonts}       
\usepackage{nicefrac}       
\usepackage{microtype}      
\usepackage[dvipsnames]{xcolor}         

\usepackage{graphicx}
\usepackage{amsmath}
\usepackage{amsthm}

\usepackage{settings_preprint}
\usepackage{authblk}

\author[a]{Lucas De Lara\thanks{E-mail: lucas.de-lara@univ-lorraine.fr}}
\author[a]{Mathis Deronzier}
\author[b]{Alberto Gonz\'alez-Sanz}
\author[a]{Virgile Foy}

\affil[a]{Institut de Mathématiques de Toulouse, Université Paul Sabatier}
\affil[b]{Department of Statistics, Columbia University}

\date{}

\title{On the Nonconvexity of Push-Forward Constraints and Its Consequences in Machine Learning\footnote{Published in SIMODS: \url{https://epubs.siam.org/doi/abs/10.1137/24M1645036}}}

\begin{document}

\maketitle

\begin{abstract}
The push-forward operation enables one to redistribute a probability measure through a deterministic map. It plays a key role in statistics and optimization: many learning problems (notably from optimal transport, generative modeling, and algorithmic fairness) include constraints or penalties framed as push-forward conditions on the model. However, the literature lacks general theoretical insights on the (non)convexity of such constraints and its consequences on the associated learning problems. This paper aims at filling this gap. In the first part, we provide a range of sufficient and necessary conditions for the (non)convexity of two sets of functions: the maps transporting one probability measure to another and the maps inducing equal output distributions across distinct probability measures. This highlights that for most probability measures, these push-forward constraints are not convex. In the second part, we show how this result implies critical limitations on the design of convex optimization problems for learning generative models or groupwise fair predictors. This work will hopefully help researchers and practitioners have a better understanding of the critical impact of push-forward conditions onto convexity.
\end{abstract}

\textbf{Keywords:} push-forward, machine learning, convexity, optimal transport, generative modeling, fairness

\section{Introduction}
Most penalties promoting group-level fairness in machine learning are nonconvex. Analogously, generative-modeling optimization problems are almost never convex even in function space. This article provides a common mathematical explanation: \emph{push-forward constraints are generally nonconvex, and no convex loss can quantify the deviation to a nonconvex subset.}

Given a Borel probability measure $P$ on $\R^d$ and a measurable function $f : \R^d \to \R^p$, the \emph{push-forward measure} of $P$ by $f$ is defined as $f_\sharp P := P \circ f^{-1}$. This operation describes the redistribution of the mass from $P$ through the deterministic allocation $f$ and plays an increasingly important role in statistics and machine learning. Notably, both optimal transport \citep{monge1781memoire} and generative modeling \citep{goodfellow2014generative, kingma2014auto, rezende2015variational} address the computation of a map $f$ satisfying $f_\sharp P = Q$ for two probability measures $P$ and $Q$. Additionally, many concepts of groupwise algorithmic fairness can be framed as finding a predictor $f$ such that $f_\sharp P = f_\sharp Q$ \citep{dwork2012fairness, hardt2016equality}. However, while convexity is critical to design statistically and numerically sound learning problems \citep{hjort1993asymptotics,bubeck2015convex}, the literature has little analyzed the (non)convexity of such constraints and its consequences in machine learning.

Our paper aims at filling this gap. We refer to functions $f$ such that $f_\sharp P = Q$ as \emph{transport maps} between $P$ and $Q$, and to functions $f$ such that $f_\sharp P = f_\sharp Q$ as \emph{equalizing maps} between $P$ and $Q$. In a first time, we thoroughly study these sets of functions, notably proving that they are most often not convex. In a second time, we address the practical relevance of the first part to understand the limitations of popular machine-learning tasks. Our reasoning rests on an overlooked result from convex analysis: there is no convex loss quantifying the deviation from a nonconvex constraint. We show how this generally renders unfeasible the design of convex learning problems involving push-forward constraints or penalties on the model, such as generative modeling and fair learning. 

As such, this theoretical work has a practical interest. While the first part can be seen as a stand-alone mathematical contribution that sheds a fresh light on the push-forward operation, it crucially provides guidance on what can(not) be achieved in generative modeling and algorithmic fairness through the second part. Concretely, we hope that this paper will not only provide a better understanding on measure transportation, but also save time for researchers and practitioners struggling to design convex learning problems.\footnote{We emphasize that the motivation for this work comes precisely from failed attempts on our side to find a convex penalty for statistical parity. This led us to identify specific cases where this was impossible and then to develop a general interpretation of this phenomenon.}

\paragraph{Outline} The rest of the paper is organized as follows.
\begin{itemize}
    \item \Cref{sec:prelim} furnishes the necessary background on the push-forward operation.
    \item \Cref{sec:shape} elucidates the convexity of the sets of transport maps (\cref{sec:transport}) and equalizing maps (\cref{sec:equalizer}) by proving that they are generally not convex.
    \item \Cref{sec:learning} first reminds that a convex loss is minimal on a convex set, and details the consequences of this result on the design of convex minimization programs (\cref{sec:convexity}). Then, it applies this framework to explain why the machine-learning problems for generative modeling (\cref{sec:generative}) and group fairness (\cref{sec:parity}), which involve (generally nonconvex) push-forward constraints, cannot be convex.
    \item \Cref{sec:recover} proposes two directions to recover convexity in such machine-learning tasks: weakening or strengthening the constraint (\cref{sec:changing}), substituting the deterministic push-forward map by a random coupling (\cref{sec:model}).
    \item  \Cref{sec:stump} focuses on equalizing maps between discrete measures, which requires specific notations.
\end{itemize}

\section{Preliminaries}\label{sec:prelim}

This preliminary section introduces the basic notation and definitions that will be used throughout the paper and provides basic knowledge on the push-forward operation.

\subsection{Notations and definitions}

Let $d,p \geq 1$ be two integers, and let $\G$ be the most general class of functions we consider in this work: the class of Borel measurable functions from $\R^d$ to $\R^p$. The other key objects are the measures on Euclidean spaces. We denote by $\norm{\cdot}$ the Euclidean norm regardless of the dimension. We refer to $\M(\R^d)$ and $\M^+(\R^d)$ as, respectively, the set of Borel measures on $\R^d$ and the set of nonnegative Borel measures on $\R^d$. Additionally, we define $\M_s(\R^d) := \{ \mu \in \M(\R^d) \mid \mu(\R^d)=s \}$ for $s \in \R$, and $\M^+_s(\R^d) := \{ \mu \in \M^+(\R^d) \mid \mu(\R^d)=s \}$ for $s \geq 0$. As such, $\mathcal{P}(\R^d) := \M^+_1(\R^d)$ is the set of probability measures on $\R^d$. For $f \in \G$, and $\mu \in \M(\R^d)$, we call $f_\sharp \mu := \mu \circ f^{-1}$ the \emph{push-forward measure} of $\mu$ by $f$. If $P \in \mathcal{P}(\R^d)$, note that $f_\sharp P$ is simply the probability law of the random variable $f(X)$ when the law of the random variable $X$ is $P$.

We denote by $\delta_x$ the \emph{Dirac measure} at a given point $x$, and by $\ell_d$ the Lebesgue measure on $\R^d$. A measure $\mu \in \M(\R^d)$ is \emph{absolutely continuous} with respect to a measure $\nu \in \M(\R^d)$, written as $\mu \ll \nu$, if for any Borel set $E \subseteq \R^d$, $(\nu(E)=0 \implies \mu(E)=0)$. We say that $\mu$ is \emph{continuous} if for any $x \in \R^d$, $\mu(\{x\})=0$. Two measures $\mu,\nu \in \M(\R^d)$ are \emph{singular} if there exists a partition $\{A,B\}$ of $\R^d$ such that $\mu(B')=\nu(A')=0$ for all Borel sets $A'\subseteq A$  and $B'\subseteq B$. Every $P \in \mathcal{P}(\R^d)$ can be written as $P = P_c + P_{\delta}$, where $P_c$ and $P_{\delta}$ are two singular measures of $\M^+(\R^d)$, such that $P_c$ is continuous and $P_{\delta}$ is a discrete (or pure point) measure.

For $P \in \mathcal{P}(\R^d)$, two functions $f,g \in \G$ are \emph{$P$-almost-everywhere equal} if $P(\{x \in \R^d \mid f(x)=g(x)\})=1$, which we write as $f \aeeq{P} g$. Then, for any $f \in \G$ and $P \in \mathcal{P}(\R^d)$, we define $\{f\}_P := \{ g \in \G \mid f \aeeq{P} g\}$. It describes a family of functions that cannot be distinguished by $P$. Throughout, we also consider a probability space $(\Omega,\Sigma,\P)$ that serves to define random variables. The \emph{law} with respect to $\P$ of any random variable or vector $X$ defined on $\Omega$ is denoted by $\mathbb{L}(X) := X_\sharp \P$, while its \emph{expectation} is denoted by $\E[X] := \int_\Omega X(\omega) \mathrm{d}\P(\omega)$. Additionally, whenever it is well-defined, $\mathbb{L}(X \mid E)$ refers to the \emph{law of $X$ conditional to $E \in \Sigma$}.

A crucial concept for the second part of the paper is \emph{convexity}. A \emph{set} $\C$ is convex if for any $u,v \in \C$ and $0<t<1$ the element $(1-t)u + tv$ belongs to $\C$. A real-valued \emph{function} $\L$ defined on a convex set $\C$ is convex if for any $u,v \in \C$ and $0<t<1$, $\L((1-t)u + tv) \leq (1-t)\L(u) + t\L(v)$. Note that for any $f \in \G$ and $P \in \mathcal{P}(\R^d)$, $\{f\}_P$ is convex.

\subsection{Push-forward calculus}

We formalize a series of elementary calculus rules for the push-forward operator that will be frequently used in the proofs of our main results. They directly follow from the definition of the push-forward measure.

\begin{proposition}[basic push-forward calculus]
Let $f,g \in \G$, $\mu,\nu \in \M(\R^d)$, and $s \in \R$. The following properties hold:
\begin{itemize}
    \item[(i)] the function $f_\sharp :\M(\R^d) \to \M(\R^p) $ is a linear map;
    \item[(ii)] $(\mu \in \M_s(\R^d) \implies f_\sharp \mu \in \M_s(\R^p))$ and $(\mu \in \M^+_{\abs{s}}(\R^d) \implies f_\sharp \mu \in \M^+_{\abs{s}}(\R^p))$;
    \item[(iii)] $\int_{\R^p} h \, \mathrm{d}(f_\sharp \mu) = \int_{\R^d} (h \circ f) \, \mathrm{d}\mu$ for every measurable function $h : \R^p \to \R$;
    \item[(iv)] $\psi_\sharp(f_\sharp \mu) = (\psi \circ f)_\sharp \mu$ for every measurable function $\psi: \R^p \to \R^k$, where $k \geq 1$ is an integer;
    \item[(v)] if $\mu \in \mathcal{P}(\R^d)$, then $(g \aeeq{\mu} f \implies g_\sharp \mu = f_\sharp \mu)$.
    \end{itemize}
\end{proposition}

Basically, a push-forward operation does not change the absolute value nor the sign of the mass: it only changes its location. The next section dives into more advanced push-forward calculus by examining sets of functions satisfying a mass-preservation constraint.

\section{The shape of push-forward constraints}\label{sec:shape}

This section focuses on clarifying the sets of transport maps and equalizing maps, by providing various necessary or sufficient conditions on the cardinality and convexity of these sets.

\subsection{Transport maps}\label{sec:transport}

Let $P \in \mathcal{P}(\R^d)$ and $Q \in \mathcal{P}(\R^p)$ be two probability measures. We consider the set of functions pushing-forward $P$ to $Q$, which we refer to as the \emph{transport maps} or \emph{measure-preserving maps} between $P$ and $Q$. Formally, we define
\[
\T(P,Q) := \{ f \in \G \mid f_\sharp P = Q\}.\footnote{This definition can naturally be extended to any measures $P \in \M(\R^d)$ and $Q \in \M(\R^p)$ with the same sign and the same total mass.}
\]
As detailed later in \cref{sec:generative}, the constraint described by this set plays a fundamental role in so-called \emph{push-forward generative modeling} \citep{salmona2022can}, where one aims at generating $Q$ by a deterministic function $f$ pushing-forward a distribution $P$ typically satisfying $d \ll p$. It also corresponds to the admissible solutions of the Monge formulation of optimal transport \citep{monge1781memoire}, which looks for the elements in $\T(P,Q)$ minimizing a certain mass-displacement cost. 

People familiar with optimal-transport theory know that $\T(P,Q)$ can be empty and nonconvex (which is often mentioned as a motivation for the well-known relaxation of \cite{kantorovich1958space}). However, the literature lacks a general understanding of what connects the cardinality and convexity of the set of transport maps to the measures $P$ and $Q$. This is precisely what we address in this subsection. For starters, let us illustrate the possible values of $\T(P,Q)$ for simple discrete measures $P$ and $Q$.
\begin{example}[simple transport maps]
We provide three examples:
\begin{itemize}
    \item[(i)] Let $P := \delta_x$ for some $x \in \R^d$, and $Q := \frac{1}{2} \delta_{y_1} + \frac{1}{2} \delta_{y_2}$ for two distinct $y_1,y_2 \in \R^p$. Remark that for any $f \in \G$, $f_\sharp P = \delta_{f(x)} \neq Q$. Thereby $\T(P,Q)$ is empty. More generally, this occurs in particular whenever the support of $Q$ is larger than the support of $P$. Note that in this case, $\T(P,Q)$ is trivially convex.
    \item[(ii)] Let $P := \delta_x$ for some $x \in \R^d$, and $Q := \delta_{y}$ for some $y \in \R^p$. It readily follows from $f_\sharp P = \delta_{f(x)}$ that $\T(P,Q) = \{f\}_P$, where $f : \{x\} \to \{y\}, x \mapsto y$. Note that in this case as well, $\T(P,Q)$ is trivially convex.
    \item[(iii)] Let $P := \frac{1}{2} \delta_{x_1} + \frac{1}{2} \delta_{x_2}$ for two distinct $x_1,x_2 \in \R^d$, and $Q := \frac{1}{2} \delta_{y_1} + \frac{1}{2} \delta_{y_2}$ for two distinct $y_1,y_2 \in \R^p$. We define the following surjective functions from $\{x_1,x_2\}$ to $\{y_1,y_2\}$: $f$ such that $f(x_1)=y_1$ and $f(x_2)=y_2$; $g$ such that $g(x_1)=y_2$ and $g(x_2)=y_1$. Note that $f$ and $g$ belong to $\T(P,Q)$. Moreover, $\frac{f(x_1)+g(x_1)}{2} = \frac{y_1+y_2}{2} \notin \{y_1,y_2\}$. Therefore, $(\frac{1}{2}f+\frac{1}{2}g) \notin \T(P,Q)$, and the set $\T(P,Q)$ is not convex.
\end{itemize}
\end{example}
We observe three configurations where the cardinality of $\T(P,Q)$ and its convexity seem intertwined. To theoretically ground this observation for general probability measures, we first study the set of functions aligning the squared Euclidean norm across $P$ and $Q$, formally defined as
\[
\T_{\norm{\cdot}^2}(P,Q):= \left\{f \in \G \mid \int \norm{f}^2 \mathrm{d}P = \int \norm{\cdot}^2 \mathrm{d}Q\right\}
\]
for $Q \in \mathcal{P}(\R^p)$ such that $\int \norm{\cdot}^2 \mathrm{d}Q < +\infty$. Critically, similarly to a sphere, $\T_{\norm{\cdot}^2}(P,Q)$ has no convex subset with more than one point, as explained below.
\begin{theorem}[nowhere convexity of the set of squared-norm matching functions]\label{thm:moment2}
Let $P \in \mathcal{P}(\R^d)$ and $Q \in \mathcal{P}(\R^p)$ be two probability measures such that $\int \norm{\cdot}^2 \mathrm{d}Q < +\infty$. For any $\F \subseteq \T_{\norm{\cdot}^2}(P,Q)$, $\F$ is either
\begin{itemize}
    \item[(i)] empty;
    \item[(ii)] equal to $\{f\}_P$ for some $f \in \G$;
    \item[(iii)] not convex.
\end{itemize}
\end{theorem}
The strategy of the proof amounts to finding a necessary condition for the convexity of $\F \subseteq \T_{\norm{\cdot}^2}(P,Q)$ that holds only if $\F$ is either empty or reduced to a singleton. It remarkably involves the equality case of the Cauchy-Schwarz inequality.
\begin{proof}[\cref{thm:moment2}]
By definition of $\T_{\norm{\cdot}^2}(P,Q)$, and since $\F \subseteq \T_{\norm{\cdot}^2}(P,Q)$, if $\F$ is convex, then for every $0<t<1$ and any $f,g \in \F$, $\int \norm{(1-t)f+tg} \mathrm{d}P = \int \norm{\cdot} \mathrm{d}Q$. After developing and simplifying using the fact that $\int \norm{f}^2 \mathrm{d}P = \int \norm{g}^2 \mathrm{d}P = \int \norm{\cdot}^2 \mathrm{d}Q < +\infty$ we obtain the following necessary condition that does not involve $t$ anymore: for every $f,g \in \F$, $\int \langle f, g \rangle \mathrm{d}P = \int \norm{\cdot}^2 \mathrm{d}Q$. Now, according to the Cauchy-Schwarz inequality and again the fact that $\int \norm{f}^2 \mathrm{d}P = \int \norm{g}^2 \mathrm{d}P = \int \norm{\cdot}^2 \mathrm{d}Q$ we obtain $\int \langle f, g \rangle \mathrm{d}P \leq \left(\int \norm{f}^2 \mathrm{d}P\right)^{1/2} \left(\int \norm{g}^2 \mathrm{d}P\right)^{1/2} = \int \norm{\cdot}^2 \mathrm{d}Q$. There is equality if and only if there exists $\alpha_{f,g} \in \R^*_+$ such that $g \aeeq{P} \alpha_{f,g} f$.

Wrapping everything up, if $\F$ is convex, then for every $f,g \in \F$ there exists a constant $\alpha_{f,g} \in \R^*_+$ such that $g \aeeq{P} \alpha_{f,g} f$. We distinguish three cases regarding the set $\F$ in the light of this condition.
\begin{enumerate}
    \item[(i)] $\F = \emptyset$: therefore it is trivially convex. 
    \item[(ii)] $\F$ is reduced to a function $P$-almost-everywhere unique: therefore it is trivially convex. It corresponds to the setting where for every $f,g \in \F$, $\alpha_{f,g}=1$.
    \item[(iii)] There exist elements $f,g \in \F$ that are not $P$-almost-everywhere equal. Assuming \emph{ad absurdum} that $\F$ is convex, the necessary condition ensures that there is a positive $\alpha_{f,g} \neq 1$ such that $g \aeeq{P} \alpha_{f,g} f$. Therefore, $\int \norm{g}^2 \mathrm{d}P = \alpha_{f,g}^2 \int \norm{f}^2 \mathrm{d}P = \alpha_{f,g}^2 \int \norm{\cdot}^2 \mathrm{d}Q \neq \int \norm{\cdot}^2 \mathrm{d}Q$. This contradicts the fact that $\int \norm{g}^2 \mathrm{d}P = \int \norm{\cdot}^2 \mathrm{d}Q$. Consequently, $\F$ is not convex.
\end{enumerate}
\end{proof}
\Cref{thm:moment2} is a strong result. It signifies that \emph{any} subclass of functions merely matching the squared Euclidean norm between two probability distributions can only take a restricted number of shapes: it is basically either trivial or nonconvex. Because $\T(P,Q) \subseteq \T_{\norm{\cdot}^2}(P,Q)$ (a transport map matches all moments, not just one) the following corollary holds.
\begin{corollary}[nonconvexity of the set of transport maps]\label{cor:transport}
Let $P \in \mathcal{P}(\R^d)$ and $Q \in \mathcal{P}(\R^p)$ be two probability measures such that $\int \norm{\cdot}^2 \mathrm{d}Q < +\infty$. Then, $\T(P,Q)$ is either
\begin{itemize}
    \item[(i)] empty;
    \item[(ii)] equal to $\{f\}_P$ for some $f \in \G$;
    \item[(iii)] not convex.
\end{itemize}
\end{corollary}
\begin{proof}
Recall that two probability measures $\mu$ and $\nu$ are equal only if $\int h \mathrm{d}\mu = \int h \mathrm{d}\nu$ for every measurable function $h$. Therefore, if $f \in \T(P,Q)$, then $\int (h \circ f) \mathrm{d}P = \int h \mathrm{d}Q$ for every measurable function $h$. In particular, the integral equality must be true for $h := \norm{\cdot}^2$, leading to $f \in \T_{\norm{\cdot}^2}(P,Q)$. This shows that $\T(P,Q) \subseteq \T_{\norm{\cdot}^2}(P,Q)$. \Cref{thm:moment2} concludes the proof.  
\end{proof}
While we find it interesting to exploit the squared-norm alignment to prove \cref{cor:transport} through \cref{thm:moment2}, it comes at the price of a moment assumption on $Q$. We leave the question of its necessity for further research.

Remarkably, \cref{cor:transport} automatically converts knowledge on the cardinality of $\T(P,Q)$ into information on its convexity. More precisely, if $\T(P,Q)$ contains at least two elements that cannot be distinguished by $P$, it is not convex. In light of this result, we determine the convexity of $\T(P,Q)$ for standard probability measures $P$ and $Q$ by specifying its cardinality. The proposition below addresses the case where $P$ is continuous.

\begin{proposition}[transport maps for a continuous source measure]\label{prop:continuous}
Let $P \in \mathcal{P}(\R^d)$ be continuous, and $Q \in \mathcal{P}(\R^p)$ not be a Dirac measure. Then, $\T(P,Q)$ contains an uncountable number of functions that are two-by-two not $P$-almost-everywhere equal.\footnote{We say that the elements of an indexed set $\{x_i\}_{i \in I}$ are \emph{two-by-two distinct} if $x_i \neq x_j$ for every $i,j \in I$. This serves to emphasize the difference with: there exist $i,j \in I$ such that $x_i \neq x_j$.}    
\end{proposition}
The proof distinguishes three situations: if $Q$ is continuous, if $Q$ is discrete, if $Q$ has continuous and discrete parts. Whatever the case, the general idea is to first send $P$ onto $U$, the uniform probability measure on $[0,1]$, and then to exhibit an uncountable number of redistributions from $U$ to $Q$.
\begin{proof}[\cref{prop:continuous}]
For starters, let us specify the objects that will be common to all parts of the proof. Since $P$ is continuous there exists according to \citep[Theorem 17.41]{kechris2012classical} a bijective measurable function $T_P : \R^d \to [0,1]$ such that ${T_P}_\sharp P = U$. Next, we define an uncountable number of allocation from $U$ to $U$. More precisely, we construct the parametric family of functions $\{\xi_a\}_{a \in [0,1)}$ by
\[
\xi_a(u) := \begin{cases}
		u+a, & \text{if $u \in [0, 1-a)$}\\
            u-1+a, & \text{if $u \in [1-a, 1]$}.
		 \end{cases}
\]
For every $a \in [0,1)$, ${\xi_a}_\sharp U = U$ and therefore $(\xi_a \circ T_P)_\sharp P = U$. Moreover, for every distinct $a,a' \in [0,1)$, $\xi_a(u) \neq \xi_{a'}(u)$ for all $u \in [0,1]$. This means that we possess an uncountable collection of distinct allocation from $P$ to $U$. Now, the crucial question is how to send the reallocated mass from $U$ to $Q$. 

In a first time, we assume that $Q$ is also continuous. According to \citep[Theorem 17.41]{kechris2012classical} again, there exists a bijective measurable function with measurable inverse $T_Q : \R^p \to [0,1]$ such that ${T_Q}_\sharp Q = U$. Then, let us define the family $\{f_a\}_{a \in [0,1)}$ by $f_a := T^{-1}_Q \circ \xi_a \circ T_P$. Crucially, it follows from the push-forward relationships that $\{f_a\}_{a \in [0,1)} \subseteq \T(P,Q)$. Moreover, recall that for every distinct $a,a' \in [0,1)$, $\xi_a(u) \neq \xi_{a'}(u)$ for all $u \in [0,1]$. This implies by the injectivity of $T_Q$ that for every $x \in T^{-1}_P([0,1])$, $f_a(x) \neq f_{a'}(x)$. Since $P(T^{-1}_P([0,1]))=U([0,1])=1$, the functions in $\{f_a\}_{a \in [0,1)}$ are two-by-two not $P$-almost-everywhere equal. Noting that $\{f_a\}_{a \in [0,1)}$ is uncountable permits us to conclude.

In a second time, we assume that $Q$ is discrete but different from a single Dirac. More precisely, for $m \geq 2$ possibly equal to $+\infty$ we write $Q := \sum^m_{j=1} \beta_j \delta_{y_j}$, where the $\{\beta_j\}^m_{j=1}$ are probability weights and the $\{y_j\}^m_{j=1}$ are two-by-two-distinct elements of $\R^p$. Next, we set $\sigma : [m] \to \{y_j\}^m_{j=1}, j \mapsto y_j$, such that by defining $Q_1 := \sum^m_{j=1} \beta_j \delta_j \in \mathcal{P}(\R)$ we have $\sigma_\sharp Q_1 = Q$. Moreover, we write $T^\dagger_Q$ for the generalized inverse distribution function of the univariate discrete measure $Q_1$. It satisfies ${T^\dagger_Q}_\sharp U = Q_1$. We define the family $\{f_a\}_{a \in [0,1)}$ by $f_a := \sigma \circ T^\dagger_Q \circ \xi_a \circ T_P.$ As before, it follows from the push-forward relationships that $\{f_a\}_{a \in [0,1)} \subseteq \T(P,Q)$. To see that the elements of this family are distinguishable under $P$, recall that for every distinct $a,a' \in [0,1)$, $f_a(x) = f_{a'}(x)$ if and only if $T^\dagger_Q \circ \xi_a \circ T_P(x) = T^\dagger_Q \circ \xi_{a'} \circ T_P(x)$. Because $Q_1$ is not a Dirac, there exists $u_0 \in (0,1)$ such that for every $u<u_0<u'$, $T^\dagger_Q(u) \neq T^\dagger_Q(u')$. Additionally, note that there exists a nontrivial interval $I_{a,a'} \subseteq [0,1]$ such that for any $u \in I_{a,a'}$, $\xi_a(u) < u_0 < \xi_{a'}(u)$, or for any $u \in I_{a,a'}$, $\xi_{a'}(u) < u_0 < \xi_{a}(u)$. As a consequence, for any $x \in T^{-1}_P(I_{a,a'})$, $f_a(x) \neq f_{a'}(x)$. Thereby, $f_a$ and $f_{a'}$ are not $P$-almost-everywhere equal for $a \neq a'$ since $P(T^{-1}_P(I_{a,a'})) = U(I_{a,a'})>0$.

In a third time, we address the case where $Q$ has a nonzero continuous part and a nonzero discrete part. Here again, we rely on $U$ and $T_P$ as previously defined. We decompose $Q$ as $Q = Q_c + Q_\delta$, where $Q_c \in \M^+(\R^p)$ is continuous and $Q_\delta \in \M^+(\R^p)$ is discrete, and we write $q_c := Q_c(\R^p)$ to $q_\delta := Q_\delta(\R^p)$. By assumption, $0<q_c=1-q_\delta<1$. We also divide $U$ into $U = U_c + U_\delta$, where $U_c$ and $U_\delta$ are the Lebesgue measures on, respectively, $[0,q_c)$ and $[q_c,1]$. Using the conclusions from the first two parts of the proof, we know that there exists a family of measurable functions $\{g_{\delta,a}\}_{a \in [0,1)}$ from $[q_c,1]$ to $\R^p$ that are two-by-two distinct on an interval and such that ${g_{\delta,a}}_\sharp U_\delta = Q_\delta$. Additionally, there exists a bijective measurable function $g_c : [0,q_c) \to \R^p$ such that ${g_c}_\sharp U_c = Q_c$. Then, we define the family $\{g_a\}_{a \in [0,1)}$ by $g_a(u) := g_c(u) \mathbf{1}_{\{u < q_c\}} + g_{\delta,a}(u) \mathbf{1}_{\{u \geq q_c\}}$, which is composed of functions that are not $U$-almost-everywhere equal and satisfy
\[
{g_a}_\sharp U = {g_a}_\sharp U_c + {g_a}_\sharp U_\delta = {g_c}_\sharp U_c + {g_{\delta,a}}_\sharp U_\delta = Q_c + Q_\delta = Q.
\]
Finally, we conclude the proof by defining the family $\{f_a\}_{a \in [0,1)}$ as $f_a := g_a \circ T_P$. It verifies $\{f_a\}_{a \in [0,1)} \subseteq \T(P,Q)$, consists of functions that are two-by-two not $P$-almost-everywhere equal, and is uncountable.
\end{proof}

\begin{remark}[existence and uniqueness of monotone push-forward maps]
A famous result of \cite{mccann1995} states that if $d=p$ and $P \ll \ell_d$ (which is more specific than being continuous), then among the infinity of transport maps from $P$ to $Q$, there exists a $P$-almost-everywhere unique function $f$ that can be written as the gradient of a convex function. Interestingly, being the gradient of a convex function generalizes the notion of monotonic functions to dimensions higher than one. Thereby, this map can be seen as the canonical redistribution from $P$ to $Q$. In particular, if $d=1$, then $f = F^{-1}_Q \circ F_P$, where $F_P$ and $F_Q$ are the cumulative distribution functions of, respectively, $P$ and $Q$. See \citep{hallin2021distribution} for an extension to $d>1$.  
\end{remark}

Another classical scenario, particularly relevant in statistics, concerns transport maps between empirical measures. Empirical probability measures drawn from continuous probability measures are almost-surely uniform finitely supported measures, and thereby apply to the next proposition.
\begin{proposition}[transport maps between uniform finitely supported measures]\label{prop:empirical}
Let $n,m \geq 1$ be two integers and $\{x_i\}^n_{i=1} \subset \R^d$ and $\{y_j\}^m_{j=1} \subset \R^p$ be composed of two-by-two distinct elements. If $P := \frac{1}{n} \sum^n_{i=1} \delta_{x_i}$ and $Q := \frac{1}{m} \sum^m_{j=1} \delta_{y_j}$, then
\begin{itemize}
    \item[(i)] if $n<m$, then $\T(P,Q) = \emptyset$;
    \item[(ii)] if $n=m$, then $\T(P,Q)$ contains exactly $n!$ functions that are two-by-two not $P$-almost everywhere equal;
    \item[(iii)] if $n > m$ and $m$ does not divide $n$, then $\T(P,Q) = \emptyset$;
    \item[(iv)] if $n > m$ and $m$ divides $n$, then $\T(P,Q)$ contains at least two functions that are not $P$-almost-everywhere equal.  
\end{itemize}
\end{proposition}
\begin{proof}[\cref{prop:empirical}]
We prove each item separately.
\begin{itemize}
    \item[(i)] If $n < m$, then there are no surjections from $\{x_i\}^n_{i=1}$ to $\{y_j\}^m_{j=1}$. Therefore, there is no transport map from $P$ to $Q$.
    \item [(ii)] If $n=m$, then there are exactly $n!$ surjections from  $\{x_i\}^n_{i=1}$ to $\{y_j\}^m_{j=1}$. Since all probability weights of $P$ and $Q$ are equal to $1/n$, this entails that $\T(P,Q)$ is composed of $n!$ distinct functions up to $P$-negligible sets.
    \item[(iii)] We prove this point by contrapositive. Assuming that $\T(P,Q)$ is not empty, there exists $f \in \G$ such that $f_\sharp P = Q$. Thereby, the definition of $P$ and $Q$ implies for every $1 \leq j\leq m$ the equality $\frac{k_j}{n}  = \frac{1}{m}$, where $k_j$ is the cardinality of $\{i \in [n] \mid f(x_i)=y_j\}$. This means in particular that $m$ divides $n$.
    \item[(iv)] Suppose now that $m$ divides $n$, so that there exists an integer $r \geq 1$ satisfying $n=mr$. Then, let $\{I_k\}^m_{k=1}$ be a partition of $[n]$ such that the cardinality of each $I_k$ for $k \in [m]$ is $r$. We define $f, g \in \G$ such that
    \[
    f(x_i)= y_j\ \text{if } i\in I_j,
    \hspace{1cm}
    g(x_i)=
    \begin{cases}
    y_1 & \mbox{if } i \in I_2\\
    y_2 & \mbox{if } i \in I_1\\
    y_j & \mbox{if } i\in  I_j, i \notin \{1,2\}\\
    \end{cases}.
    \]
    They satisfy $f_\sharp P=Q$ and $g_\sharp P = Q$ while $P(\{x \in \R^d \mid f(x) \neq g(x)\}) > 0$.
\end{itemize}
\end{proof}

To sum up, through \cref{cor:transport}, it follows from \cref{prop:continuous,prop:empirical} that $\T(P,Q)$ is trivial for very specific $P$ and $Q$ and nonconvex otherwise. In the next subsection, we tackle a similar clarification work for the set of equalizing maps.

\subsection{Equalizing maps}\label{sec:equalizer}

Let $P, Q \in \mathcal{P}(\R^d)$ be two probability measures. We turn to the set of functions transforming $P$ and $Q$ into a same arbitrary measure in $\mathcal{P}(\R^p)$, which we refer to as the \emph{equalizing maps} between $P$ and $Q$. Formally, we define
\[
\mathcal{E}(P,Q) := \{ f \in \G \mid f_\sharp P = f_\sharp Q\}.\footnote{Similarly to transport maps, this definition can naturally be extended to any measures $P,Q \in \M(\R^d)$ with the same sign and the same total mass.}
\]
The definition above is motivated by algorithmic-fairness problems, where one typically tries to design models producing the same distributions of outputs across distinct protected groups. We detail this connection in \cref{sec:parity}.

In contrast to the set of transport maps, which became more apprehensible due to its key role in the intensely studied optimal-transport theory, the set of equalizing maps lacks basic insights. Let us begin the clarification by presenting trivial facts.

\begin{proposition}[basic properties]
For any $P,Q \in \mathcal{P}(\R^d)$,
\begin{enumerate}
    \item[(i)] $\mathcal{E}(P,Q) = \mathcal{E}(Q,P)$;
    \item[(ii)] if $P=Q$, then $\mathcal{E}(P,Q) = \G$;
    \item[(iii)] $\mathcal{E}(P,Q) \neq \emptyset$, as it contains in particular all the constant functions;
    \item[(iv)] for any $f \in \mathcal{E}(P,Q)$, $\psi \circ f \in \mathcal{E}(P,Q)$ for any measurable $\psi : \R^p \to \R^p$.\footnote{It directly follows from (iv) that $\mathcal{E}(P,Q)$ is path connected for any $P,Q \in \mathcal{P}(\R^d)$.} 
\end{enumerate}
\end{proposition}

As a preliminary analysis, let us resolve the convexity of $\mathcal{E}(P,Q)$ for simple discrete measures $P$ and $Q$.
\begin{example}[simple equalizing maps]\label{ex:equalizing}
We provide three examples:
\begin{itemize}
    \item[(i)] Let $P := \frac{1}{2} \delta_{x_1} + \frac{1}{2} \delta_{x_2}$ and $Q := \frac{1}{2} \delta_{y_1} + \frac{1}{2} \delta_{y_2}$ for distinct $x_1,x_2,y_1,y_2 \in \R^d$. We set two distinct $z_1,z_2 \in \R^p$ and write $R := \frac{1}{2} \delta_{z_1} + \frac{1}{2} \delta_{z_2}$. Then, we define $f : \R^d \to \R^p$ such that $f(x_1) = f(y_1) = z_1$ and $f(x_2)=f(y_2)=z_2$. Similarly, we define $g : \R^d \to \R^p$ such that $g(x_1) = g(y_2) = z_1$ and $g(x_2)=g(y_1)=z_2$. Note that $f_\sharp P = f_\sharp Q = g_\sharp P = g_\sharp Q = R$. Moreover, $(\frac{1}{2} f + \frac{1}{2} g)_\sharp P = R$ whereas $(\frac{1}{2} f + \frac{1}{2} g)_\sharp Q \neq R$ since $(\frac{1}{2} f + \frac{1}{2} g)(y_1) = \frac{z_1+z_2}{2} \notin \{z_1,z_2\}$. Therefore, $\mathcal{E}(P,Q)$ is not convex.
    \item[(ii)] Let $P := \frac{1}{2} \delta_{x_1} + \frac{1}{2} \delta_{x_2}$ and $Q := \frac{1}{3} \delta_{y_1} + \frac{2}{3} \delta_{y_2}$ for distinct $x_1,x_2,y_1,y_2 \in \R^d$. Note that for any $f \in \mathcal{E}(P,Q)$, the push-forward of $P$ and $Q$ is supported by one or two points. The one-point case corresponds to the almost-everywhere-constant functions. However, due to incompatible masses, there is no output measure $R$ supported by two points such that $f_\sharp P = f_\sharp Q = R$. Therefore, $\mathcal{E}(P,Q)$ narrows down to the functions that are constant $P+Q$-almost everywhere, which is a convex set.
    \item[(iii)] Let $P := \frac{1}{3} \delta_{x_1} + \frac{2}{3} \delta_{x_2}$ and $Q := \frac{1}{3} \delta_{y_1} + \frac{2}{3} \delta_{y_2}$ for distinct $x_1,x_2,y_1,y_2 \in \R^d$. As in (ii), for any $f \in \mathcal{E}(P,Q)$, the push-forward of $P$ and $Q$ are supported by one or two points, with the one-point case corresponding to the almost-everywhere-constant functions. In the two-point case, the output measure necessarily has the form $R := \frac{1}{3} \delta_{z_1} + \frac{2}{3} \delta_{z_2}$ for two distinct $z_1,z_2 \in \R^p$, and the admissible $f$ are constrained to send $\{x_1,y_1\}$ to $\{z_1\}$ and $\{x_2,y_2\}$ to $\{z_2\}$. Checking the different convex combinations shows that $\mathcal{E}(P,Q)$ is convex but not reduced to constant functions.
\end{itemize}
\end{example}

These examples do not highlight a universal classification as explicit as the \say{trivial versus nonconvex} from \cref{cor:transport}, which focused on transport maps. Nevertheless, we can identify sharp conditions on the (non)convexity of the set of equalizing maps for specific classes of probability measures $P$ and $Q$. The proposition below fully determines the convexity of $\mathcal{E}(P,Q)$ when $P,Q$ are both atomless.

\begin{proposition}[equalizing maps between continuous measures]\label{prop:equalizer}
Let $P,Q \in \mathcal{P}(\R^d)$ be two continuous probability measures such that $P \neq Q$. Then, $\mathcal{E}(P,Q)$ is not convex.
\end{proposition}
Up to a subtlety when the supports of $P$ and $Q$ intersect each other, the key idea amounts to equally dividing the mass in $P$ and $Q$ to recover the same configuration as the (i) case from \cref{ex:equalizing}.
\begin{proof}[\cref{prop:equalizer}]
We denote by $\varphi_P$ and $\varphi_Q$ two density functions of, respectively, $P$ and $Q$ with respect to $P+Q$. Then, we define $\S_P := \{x \in \R^d \mid \varphi_P(x)>0\}$ and $\S_Q := \{x \in \R^d \mid \varphi_Q(x)>0\}$. In a first time, we assume that $\S_P \cap \S_Q = \emptyset$. Let $U \in \mathcal{P}(\R)$ be the uniform measure on $[0,1]$. Since $P$ and $Q$ are continuous, there exist two measurable maps $T_P$ and $T_Q$ such that ${T_P}_\sharp P = {T_Q}_\sharp Q = U$ as a consequence of \cite[Theorem 17.41]{kechris2012classical}. Next, we set $z,z \in \R^p$ such that $z \neq z'$ and define two functions: $\psi_1 : \R \to \R^p$ by $\psi_1(x) := \mathbf{1}_{\{x < 1/2\}} z + \mathbf{1}_{\{x \geq 1/2\}} z'$ and $\psi_2 : \R \to \R^p$ by $\psi_2(x) := \mathbf{1}_{\{x < 1/2\}} z' + \mathbf{1}_{\{x \geq 1/2\}} z$. Observe that ${\psi_1}_\sharp U = {\psi_2}_\sharp U = \frac{1}{2} \delta_z + \frac{1}{2} \delta_{z'}$. Finally, using the fact that $\S_P \cap \S_Q = \emptyset$, we define $f, g \in \G$ as,
\[
f(x) = \begin{cases}
		\psi_1 \circ T_P(x) & \text{if $x \in \S_P$},\\
            \psi_1 \circ T_Q(x) & \text{if $x \in \S_Q$},\\
            0 & \text{otherwise},
		 \end{cases}
\ \ \ \ \text{and} \ \ \ \
g(x) = \begin{cases}
		\psi_1 \circ T_P & \text{if $x \in \S_P$},\\
            \psi_2 \circ T_Q, & \text{if $x \in \S_Q$},\\
            0 & \text{otherwise}.
		 \end{cases}
\]
They satisfy $f_\sharp P = f_\sharp Q = g_\sharp P = g_\sharp Q$, and $(\frac{1}{2}f+\frac{1}{2}g)_\sharp P = f_\sharp P = \frac{1}{2} \delta_z + \frac{1}{2} \delta_{z'}$. Moreover, for any $x \in \S_Q$, $(\frac{1}{2}f+\frac{1}{2}g)(x) = \frac{z+z'}{2} \notin \{ z,z'\}$. Therefore, $(\frac{1}{2}f+\frac{1}{2}g)_\sharp Q \neq (\frac{1}{2}f+\frac{1}{2}g)_\sharp P$ and $\mathcal{E}(P,Q)$ is not convex.

In a second time, we address the more general case where possibly $\S_P \cap \S_Q \neq \emptyset$. We define three additional measures: $\min(P,Q) \in \M^+(\R^d)$ with density $\min\{\varphi_P,\varphi_Q\}$ with respect to $P+Q$, $P':=P-\min(P,Q) \in \M^+(\R^d)$ and $Q':= Q-\min(P,Q) \in \M(\R^d)$. Note that $P'$ and $Q'$ have the same total mass and are still continuous. More precisely, $P'(\R^d) = P(\R^d) - \min(P,Q)(\R^d) = 1 - \min(P,Q)(\R^d) = Q(\R^d) - \min(P,Q)(\R^d) = Q'(\R^d) > 0$. We write $\gamma := P'(\R^d) = Q'(\R^d)$, so that $P',Q' \in \M^+_\gamma(\R^d)$. They admit $\varphi_{P'} := \left(\varphi_P - \varphi_Q\right) \mathbf{1}_{\{\varphi_P-\varphi_Q>0\}}$ and $\varphi_{Q'} := \left(\varphi_Q - \varphi_P\right) \mathbf{1}_{\{\varphi_P-\varphi_Q<0\}}$ as respective densities with respect to $P+Q$. Critically, these densities are positive on disjoint sets. Therefore, we know by the previous case that there exist $f$ and $g$ such that $f_\sharp P' = f_\sharp Q' = g_\sharp P' = g_\sharp Q'$ and $(\frac{1}{2}f+\frac{1}{2}g)_\sharp P' \neq (\frac{1}{2}f+\frac{1}{2}g)_\sharp Q'$. Let us now show the nonconvexity of $\mathcal{E}(P,Q)$. First, $f_\sharp P = f_\sharp \min(P,Q) + f_\sharp P'= f_\sharp \min(P,Q) + f_\sharp Q' = f_\sharp Q$ and by a similar computation $g_\sharp P = g_\sharp Q$. Hence, $f,g \in \mathcal{E}(P,Q)$. Second,
\begin{align*}
    \left(\frac{1}{2}f+\frac{1}{2}g\right)_\sharp P &= \left(\frac{1}{2}f+\frac{1}{2}g\right)_\sharp \min(P,Q) + \left(\frac{1}{2}f+\frac{1}{2}g\right)_\sharp P',\\
    \left(\frac{1}{2}f+\frac{1}{2}g\right)_\sharp Q &= \left(\frac{1}{2}f+\frac{1}{2}g\right)_\sharp \min(P,Q)+ \left(\frac{1}{2}f+\frac{1}{2}g\right)_\sharp Q'.
\end{align*}
It follows from $(\frac{1}{2}f+\frac{1}{2}g)_\sharp P' \neq (\frac{1}{2}f+\frac{1}{2}g)_\sharp Q'$ that $(\frac{1}{2}f+\frac{1}{2}g)_\sharp P \neq (\frac{1}{2}f+\frac{1}{2}g)_\sharp Q$. Hence, $(\frac{1}{2}f+\frac{1}{2}g) \notin \mathcal{E}(P,Q)$. Consequently, $\mathcal{E}(P,Q)$ is not convex.
\end{proof}

\begin{remark}[Not in the published version]
In the published version, \Cref{prop:equalizer} addresses the more restricted case where $P$ and $Q$ are \emph{absolutely continuous with respect to $\ell_d$}. We noticed afterward that the proof could be generalized to the continuous case by using density functions with respect to $P+Q$. We recall that for every $P,Q \in \mathcal{P}(\R^d)$, $P,Q \ll P+Q$.
\end{remark}

Additionally, in the important case where $P$ and $Q$ are both finitely supported, we can completely characterize the convexity of $\mathcal{E}(P,Q)$. For the sake of simplicity and concision, we defer this result to \cref{sec:stump} (see \cref{thm:discreteCharacConv}) since the obtained conditions are fairly intricate and involved notationwise. Basically, it asserts that there is no universal convexity of $\mathcal{E}(P,Q)$ in this discrete setting (it depends on the location, number, and probability weights of the points). A significant implication of this characterization concerns empirical measures, as made precise by the following proposition. Recall that two integers are \emph{coprime} if the only positive integer that divides both of them is 1.

\begin{proposition}[equalizing maps between uniform finitely supported measures]\label{prop:empirical_equalizing}
Let $n,m \geq 1$ be two integers and the sets $\{x_i\}^n_{i=1},\{y_j\}^m_{j=1} \subset \R^d$ be both composed of two-by-two distinct elements such that $\{x_i\}^n_{i=1} \cap \{y_j\}^m_{j=1} = \emptyset$. If $P := \frac{1}{n} \sum_{i=1}^n \delta_{x_i}$ and $Q := \frac{1}{m} \sum_{j=1}^m \delta_{y_j}$, then
\begin{itemize}
\item[(i)] if $m$ and $n$ are coprime, then $\mathcal{E}(P,Q)$ is the set of functions in $\G$ whose restrictions on $\{x_i\}_{i=1}^n \cup \{y_j\}_{j=1}^m$ are constant, which is a convex set;
\item[(ii)] if $m$ and $n$ are not coprime, then $\mathcal{E}(P,Q)$ is not convex.
\end{itemize}
\end{proposition}
According to \cref{prop:empirical_equalizing}, the set of equalizing maps between empirical measures is either trivial or nonconvex. The proof, detailed in \cref{sec:stump}, consists in checking one of the necessary conditions for convexity given by \cref{thm:discreteCharacConv}. We also refer to \cref{prop:nondisjoint} to address the case with intersecting supports.

All in all, while the possible shapes of the equalizing constraint are richer than the ones of the transport constraint, \cref{prop:equalizer,prop:empirical_equalizing} highlight that convexity remains rare and sometimes occurs only because the equalizing maps are constant (that is, trivial). In the rest of the article, we discuss the consequences of these results on machine-learning problems.

\section{Application to machine learning}\label{sec:learning}

This section illustrates the role of the sets $\T(P,Q)$ and $\mathcal{E}(P,Q)$ in popular machine-learning problems and highlights the consequences of their (non)convexity.

\subsection{Learning and convexity}\label{sec:convexity}
First, we introduce unified formulations of learning problems involving a specific condition on the models. Most machine-learning problems amount to minimizing a numerical criteria (e.g., prediction accuracy) under a constraint or a penalty (e.g., sparsity). Formally, let $\F \subseteq \G$ be a set of \emph{base} models (e.g., neural networks with a fixed architecture), and $\C \subseteq \G$ be a \emph{constraint}. In this case $\F \cap \C$ represents the set of \emph{admissible} or \emph{feasible} models. In optimization and learning problems, a loss serves to quantify the deviation of a model to some constraint. For the sake of clarity, we use the following terminology in the rest of the paper.
\begin{definition}[$\C$-loss]
Let $\C$ be a subset of $\G$. A function $\L : \G \to [0,+\infty]$ is a \emph{$\C$-loss} if for any $f \in \G$, $\L(f)=0 \iff f \in \C$.
\end{definition}
For a given loss function $\L : \G \to [0,+\infty]$, a learning problem including $\C$ in its objective fits either the \emph{constrained} optimization problem
\begin{equation}\label{eq:constrained}
    \min_{f \in \F \cap \C} \L(f)
\end{equation}
or the \emph{penalized} optimization problem
\begin{equation}\label{eq:penalized}
    \min_{f \in \F} \L(f) + \lambda \L_\C(f),
\end{equation}
where $\lambda>0$ governs a trade-off between $\L$ and a $\C$-loss $\L_\C$.

A \emph{minimization problem} is convex if both the objective to minimize and the set of feasible solutions are convex. Researchers and practitioners generally endeavor to design convex optimization problems for essentially two reasons. From an optimization viewpoint, there exist efficient numerical procedures to solve convex programs \cite{bubeck2015convex}. From a statistical viewpoint, the minimizers of empirical convex objectives often enjoy asymptotic and nonasymptotic guarantees \citep{haberman1989concavity,pollard1991asymptotics,niemiro1992asymptotics,arcones1998asymptotic,bartlett2006convexity}. On the basis on the above generic formulations, we can easily identify sufficient conditions for convexity in learning problems.
\begin{itemize}
    \item Problem~\cref{eq:constrained} is convex if $\C$ and $\F$ are convex sets (since convexity is stable under intersection) and if $\L$ is a convex function.
    \item Problem~\cref{eq:penalized} is convex if $\F$ is a convex set and if $\L$ and $\L_\C$ are convex functions.
\end{itemize}
Let us underline the specific role of $\C$, as we aim at studying examples where $\C$ is a set of transport maps or equalizing maps. The influence of $\C$ is straightforward in \cref{eq:constrained}, while it depends on a $\C$-loss $\L_\C$ in \cref{eq:penalized}. Notably, a $\C$-loss is obviously not uniquely determined by $\C$, which makes the choice of $\L_\C$ crucial to attaining convexity in \cref{eq:penalized}. But we critically emphasize that designing a convex $\C$-loss is sometimes impossible, as a consequence of a classical result of convex analysis.
\begin{theorem}[no convex loss for nonconvex constraints]\label{thm:impossible}
If $\L : \G \to [0,+\infty]$ is a convex function, then $\L^{-1}(\{0\})$ is a convex set. Therefore, if $\C \subseteq \G$ is nonconvex, then there exist no convex $\C$-loss.
\end{theorem}
\begin{proof}[\cref{thm:impossible}]
For any $f,g \in \F$ and $0<t<1$, we have by convexity and nonnegativity
\[
0 \leq \L((1-t)f + tg) \leq (1-t)\L(f) + t \L(g).
\]
Therefore, if $f$ and $g$ belong to $\L^{-1}(\{0\})$, that is, $\L(f)=\L(g)=0$, this inequality entails $\L((1-t)f + tg)=0$, that is, $(1-t)f + tg \in \L^{-1}(\{0\})$, which means that $\L^{-1}(\{0\})$ is convex.
\end{proof}
This furnishes a simple criterion to check whether a condition $\C$ can(not) be quantified by a convex loss: it suffices to verify that $\C$ itself is (not) a convex set. The consequences of this result in machine learning are significant and perhaps not well appreciated. It means that it can be in vain to look for a convex penalty.

Note that several frameworks focus on \emph{approximate} constraints rather than exact ones. For $\L_\C$ a relevant $\C$-loss and $\varepsilon \geq 0$ an approximation level, this consists in replacing $\C$ by $\C_{\L_{\C},\varepsilon} := \{f \in \G \mid \L_\C(f) \leq \varepsilon\} = \L^{-1}_\C([0,\varepsilon])$ in Problem~\cref{eq:constrained}. Crucially, we highlight that this alternative essentially does not change the (non)convexity of the problem. If $\C$ is convex, then there exists a convex $\C$-loss $\L_\C$ (like $\L_\C(f) := \operatorname{dist}(f,\C)$), and hence $\C_{\L_{\C},\varepsilon}$ is convex by being the sublevel set of a convex function. The case of a nonconvex $\C$ is not as straightforward. While there exist no convex $\C$-loss due to \cref{thm:impossible}, $\C_{\L_{\C},\varepsilon}$ could be convex for some $\varepsilon$ and $\L_\C$. However, the next proposition shows that whatever the loss function, the approximation level $\varepsilon$ must be beyond a threshold---that depends on $\C$ and $\L_\C$---to \emph{potentially} recover convexity. 
\begin{proposition}[nonconvexity of the approximations of nonconvex constraints]\label{prop:approximate} Let $\C$ be a nonconvex subset of $\G$. Then, for any $\C$-loss $\L$ there exists $\tau > 0$ such that for every $\varepsilon \in [0,\tau)$, $\L^{-1}([0,\varepsilon])$ is not convex.
\end{proposition}
\begin{proof}[\cref{prop:approximate}]
Let us assume \emph{ad absurdum} that there exists a $\C$-loss $\L$ such that for every $\tau > 0$ there exists $\varepsilon$ such that $\L^{-1}([0,\varepsilon])$ is convex. Therefore, for any integer $n \geq 1$ there exists $\varepsilon_n \in [0,1/n)$ such that $\L^{-1}([0,\varepsilon_n])$ is convex. Next, we show that $\L^{-1}(\{0\}) = \bigcap_{n \geq 1} \L^{-1}([0,\varepsilon_n])$. Clearly, $\L^{-1}(\{0\}) \subseteq \bigcap_{n \geq 1} \L^{-1}([0,\varepsilon_n])$. To prove the converse inclusion, let $f \in \bigcap_{n \geq 1} \L^{-1}([0,\varepsilon_n])$, which means that for every $n \geq 1$, $0 \leq \L(f) \leq \varepsilon_n$. Because $\lim_{n \to +\infty} \varepsilon_n = 0$, taking the limit in this inequality leads to $\L(f) = 0$, that is, $f \in \L^{-1}(\{0\})$. To conclude, note that $\bigcap_{n \geq 1} \L^{-1}([0,\varepsilon_n])$ is convex by the intersection of convex sets. Thereby, $\L^{-1}(\{0\})$ is convex, which contradicts the nonconvexity of $\C$.
\end{proof}
Therefore, given a relevant $\C$-loss $\L_\C$, the set of $\varepsilon$ such that $\C_{\L_{\C},\varepsilon}$ is convex could be empty or could range beyond a high threshold. Moreover, it may be difficult to determine and estimate this set in practice. As such, taking a sufficiently large $\varepsilon$ could simply be pointless or could excessively degrade the approximation quality.

All in all, the sufficient conditions for convexity along with \cref{thm:impossible} entail that one cannot guarantee the convexity of \cref{eq:constrained} and \cref{eq:penalized} as soon as $\C$ is not convex. This limitation extends to approximate constraints due to \cref{prop:approximate}. Next, we advance from general constraints $\C$ to specifically transport maps and equalizing maps.

\subsection{Learning under push-forward constraints or penalties}
If $\C:=\T(P,Q)$ or $\C:=\mathcal{E}(P,Q)$, then convexity depends on $(P,Q)$ according to \cref{sec:shape}. Notably, such constraints are not universally convex (i.e., convex whatever the input measures) and even actually nonconvex in many classical scenarios. In particular, this (informally) signifies through \cref{thm:impossible} that \emph{there exist no losses $\L_{\T(P,Q)}$ and $\L_{\mathcal{E}(P,Q)}$ that are convex for all measures $P$ and $Q$}, in contrast to (for instance) the mean square error, which is convex regardless of the data distribution. This is a strong limitation on the design of convex learning problems, since in typical scenarios the measures are exogenous factors. Moreover, being able to certify the convexity of $\C:=\T(P,Q)$ or, respectively, $\C:=\mathcal{E}(P,Q)$, does not even mean being able to construct a workable convex $\C$-loss for this specific case. Of course, if $\C$ is convex, then $\L_\C(f) := \operatorname{dist}(f,\C)$ or $\L_\C(f) := \mathbf{1}_{\G \setminus \C}(f)$ define convex $\C$-losses, but they require knowing $\C$ \emph{explicitly} to be computed while the role of a loss is precisely to \emph{implicitly} quantify a condition. 

In what follows, we illustrate that this setting applies to generative modeling and group-level algorithmic fairness, explaining why such problems are generally not provably convex.

\subsubsection{Generative modeling}\label{sec:generative}

Let $P \in \mathcal{P}(\R^d)$ be an source probability distribution, $Q \in \mathcal{P}(\R^p)$ be a target probability distribution with bounded support, and $D : \mathcal{P}(\R^p) \times \mathcal{P}(\R^p) \to [0,+\infty]$ be a discrepancy function between probability measures. Finding a push-forward generative model for $Q$ from $P$ amounts to finding a model $f \in \F$ such that $f_\sharp P \approx Q$, which can be achieved by solving
\begin{equation}\label{eq:generative}
    \min_{f \in \F} D(f_\sharp P, Q).
\end{equation}
This setting notably includes \emph{generative adversarial networks} (GAN) and \emph{variational auto-encoders} (VAE). Typically, the discrepancy $D$ is chosen as the Kullback-Leibler divergence \citep{goodfellow2014generative}, the Wasserstein-1 distance \citep{arjovsky2017wasserstein}, or a Sinkhorn divergence \citep{genevay2018learning}. Critically, the formulation fits \cref{eq:penalized} with $\C := \T(P,Q)$, $\L_\C(f) := D(f_\sharp P, Q)$, and $\L(f) := 0$. \emph{Thereby, \cref{eq:generative} is not universally convex.} More precisely, there are two limits to its convexity according to the sufficient conditions discussed in \cref{sec:convexity}.

First limit: the convexity of the set of feasible models $\F$. In most cases, $\F$ is a set of neural networks with fixed architecture. Note that such a set is not necessarily convex, since in particular the sum of two neural networks is generally a neural network with different depth and widths. One could work instead with a convex class of models, like linear models, but it would sacrifice the necessary inductive power of neural networks for generative tasks.

Second limit: the convexity of $\T(P,Q)$, which only holds in restricted cases according to \cref{sec:transport}. More specifically, in the GAN scenario, the true (or population) target distribution $Q$ is continuous. Therefore, \cref{prop:continuous} along with \cref{thm:impossible} ensures that \cref{eq:generative} cannot be convex---whatever the choice of $D$, \emph{even if it is convex}. This emphasizes that the convexity of $Q^\star \mapsto D(Q^\star,Q)$ (in measure space) is \emph{radically} different from the convexity of $f \mapsto D(f_\sharp P,Q)$ (in function space).\footnote{Another striking illustration of this distinction comes from \citep[Section 7.3]{ambrosio2005gradient}, which shows that the Wasserstein-2 distance is \emph{convex} in both input measures but \emph{concave} along its geodesics.} Notably, the global convergence guarantees from the original GAN paper leverage unrealistic assumptions to reframe the GAN minimization problem as a convex program in \emph{measure space} \citep[Proposition 2]{goodfellow2014generative}. Following a more statistically driven approach, one rather solves \cref{eq:generative} between some empirical measures $P_n$ and $Q_m$, respectively corresponding to an $n$-sample from $P$ and an $m$-sample from $Q$. According to \cref{prop:empirical}, $\T(P_n,Q_m)$ is convex only if $m$ does not divide $n$. Not only this is restrictive, but this would solely guarantee the existence of a convex $\C$-loss; it would not provide a closed-form expression as previously mentioned.

Last, we shall mention that the class $\F = \{f_\theta\}_{\theta \in \Theta}$ is generally parametric, with $\Theta$ included in an Euclidean space. Therefore, in practice, one does not optimize in function space but in parameter space. Several references noticed the nonconvexity of the GAN objective in parameter space, pointing out the nonconvexity of $\theta \mapsto f_\theta$ for common neural networks \citep{nagarajan2017gradient,guo2023gans}. Our analysis highlights a more structural culprit: the nonconvexity of the push-forward operation. \emph{All in all, push-forward generative modeling suffers from strong limitations to convexity at every level, making it almost never convex.}

\subsubsection{Group-level algorithmic fairness}\label{sec:parity}

We turn to the design of fair machine-learning predictors. Let $X : \Omega \to \R^{d-1}$ be a random vector representing some covariates, and $S : \Omega \to \{0,1\}$ be a random variable encoding a binary protected status (e.g., males and females) such that $0 < \P(S=1) < 1$. A function $f \in \G$ satisfies \emph{statistical parity} (also known as \emph{demographic parity} or \emph{no disparate impact}) with respect to $S$ if $f(X,S) \independent S$ \citep{dwork2012fairness}. We define $P := \mathbb{L}((X,S) \mid S=0)$ and $Q := \mathbb{L}((X,S) \mid S=1)$ the conditional probability measures of the two protected groups. In this binary case, note that statistical parity can be framed as $\mathbb{L}(f(X,S) \mid S=0) = \mathbb{L}(f(X,S) \mid S=1)$, that is, $f_\sharp P = f_\sharp Q$, namely, $f \in \mathcal{E}(P,Q)$.

Maximizing the accuracy of the model under the statistical-parity constraint corresponds to \cref{eq:constrained} with $\L(f) := \E\left[\norm{f(X,S)-Y}^2\right]$ (or any other convex loss for accuracy) and $\C := \mathcal{E}(P,Q)$ (as in \citep{gouic2020projection,evgenii2020fair}). Obtaining a trade-off between accuracy and fairness corresponds to \cref{eq:penalized} with $\L_\C$ an $\mathcal{E}(P,Q)$-loss (as in \citep{perez2017fair,risser2022tackling}).
Most often, people work with $\L_\C(f) := D(f_\sharp P, f_\sharp Q)$, where $D$ is (the power of) a distance or divergence between probability measures.

Similar to generative modeling, the convexity of a learning process involving statistical parity has two restrictions according to \cref{sec:convexity}: the convexity of $\F$ and the one of $\mathcal{E}(P,Q)$. Regarding the former, if $\F$ is a set of neural networks with fixed architecture, then the learning is not convex; if $\F$ is the parametric set of linear predictors, then the learning could be convex (but likely less efficient on complex data). Regarding the latter, the discussion from \cref{sec:convexity} along with the results from \cref{sec:equalizer} show that the learning is rarely convex, whether it be for Lebesgue absolutely continuous population measures $P,Q$ (\cref{prop:equalizer}) or their empirical counterparts (\cref{prop:empirical_equalizing}). \emph{As such, no practitioner can design a group-fair learning problem that universally guarantees convexity.} Notably, as for \cref{sec:generative}, modifying $\L_\C$ (using, for instance, a convex $D$) is pointless. Additionally, considering an $\varepsilon$-approximate version of statistical parity comes with no reliable convexity guarantees according to \cref{prop:approximate}.

Strictly, we should take into account that the models belong to $\F$ instead of $\G$. It could happen that for relevant subclasses of models $\F$, the set of feasible solutions $\mathcal{E}(P,Q) \cap \F$ is convex. In the next proposition, we show that convexity does not universally hold for the widely used class of weight-linear predictors.
\begin{proposition}[no universal convexity of the set of linear equalizing maps]\label{prop:linear}
We define the parametric class of functions $\F_{\operatorname{Lin}} := \{ x \mapsto  \theta x + \theta_0; \theta \in \R^{p \times d}, \theta \in \R^p\}$. If $d \geq 2$, then there exist $P,Q \in \mathcal{P}(\R^d)$ such that $\mathcal{E}(P,Q) \cap \F_{\operatorname{Lin}}$ is not convex.
\end{proposition}
\begin{proof}[\cref{prop:linear}]
We define two functions $f$ and $g$ in $\F_{\operatorname{Lin}}$ by $f(x_1,\ldots,x_d) := x_1$ and $g(x_1,\ldots,x_d) := - x_2$. Set $A$ and $B$ two independent and identically distributed random variables, and let $P,Q \in \mathcal{P}(\R^d)$ be the laws of, respectively, $(A,B,0,\ldots,0)$ and $(A,A,0,\ldots,0)$. Note that $f_\sharp P = f_\sharp Q = \mathbb{L}(A)$ and $g_\sharp P = g_\sharp Q = \mathbb{L}(-A)$ since $A$ and $B$ are equal in law. Moreover, $(\frac{1}{2}f+\frac{1}{2}g)_\sharp P = \delta_0$ and $(\frac{1}{2}f+\frac{1}{2}g)_\sharp Q= \mathbb{L}\left(\frac{1}{2}B-\frac{1}{2} A\right)$, which is not equal to $\delta_0$ since $A$ and $B$ are independent and not constant. Therefore, $(\frac{1}{2}f+\frac{1}{2}g)_\sharp P \neq (\frac{1}{2}f+\frac{1}{2}g)_\sharp Q$, and $\mathcal{E}(P,Q) \cap {\F_{\operatorname{Lin}}}$ is not convex.
\end{proof}
Note that the same conclusion holds for any class $\F$ such that $\F_{\operatorname{Lin}} \subseteq \F$.

\begin{remark}[Transport-map approaches to fairness]\label{rem:fairness} The classification method of \cite{gordaliza2019obtaining} achieves statistical parity with a limited loss in accuracy by optimally modifying the input data. To do so, it transports all the protected-group distributions of covariates toward a common well-chosen reference distribution. When $S$ is binary, this corresponds to computing $f_0 \in \T(P,R)$ and $f_1 \in \T(Q,R)$ with $P := \mathbb{L}(X \mid S=0)$, $Q := \mathbb{L}(X \mid S=1)$, and $R$ the Wasserstein barycenter of $\{P,Q\}$. This problem is also nonconvex in general. In contrast to the problem detailed in \cref{sec:parity}, its nonconvexity comes from the nonconvexity of $\T(P,R)$ and $\T(Q,R)$ rather than the one of $\mathcal{E}(P,Q)$.
\end{remark}

\section{Toward recovering convexity}\label{sec:recover}

As explained in the previous section, one cannot guarantee the convexity of a learning problem involving a nonconvex condition $\C$, be it as a constraint or a penalty. This section explores two directions that one can follow to recover convexity, exemplified on problems involving transport maps or equalizing maps.

\subsection{Weakening or strengthening the condition}\label{sec:changing}

The simplest approach amounts to replacing the constraint $\C$ by a different one that still captures the same principle while being convex. Let us illustrate this idea with statistical parity. In this subsection, we assume that $p=1$, so that all models $f$ are functions from $\R^d$ to $\R$.

Basically, statistical parity requires the predictions $f(X,S)$ to be independent of the protected attribute $S$. More precisely, it leverages the probabilistic notion of independence between random variables. However, other concepts of dependence could be employed. For instance, one could simply demand $\Cov(f(X,S),S)=0$, where $\Cov$ is the covariance between two random variables. Note that with the notation of \cref{sec:parity}, it corresponds to the constraint $f \mathrm{d}P = \int f \mathrm{d}Q$. This leads to a weaker definition of fairness that still limits the dependence of $f(X,S)$ to $S$ but that is convex (even linear) in the model. This is notably what \citep{zafar2017fairness,zafar2019fairness} did to obtain a convex relaxation of statistical parity in their fair-classification frameworks. Conversely, one can also reach convexity by strengthening the definition. In contrast to statistical parity, which is a distributional (or group) definition of fairness, \emph{counterfactual fairness} focuses on input-level predictions \citep{kusner2017counterfactual}. It holds when the model produces the same outputs for every input and their counterfactual counterparts \emph{had the protected status changed}. While it became famous for addressing causality rather than mere associations, another interesting aspect explored in \citep{kusner2017counterfactual,delara2024transport} lies in the fact that (under standard assumptions) it implies statistical parity by being its input-scale counterpart. Moreover, counterfactual fairness represents a convex constraint that can be quantified through a convex loss \citep{russell2017when, delara2024transport}. It can thereby be used as a convex restriction of statistical parity.

This illustrates that weakening or strengthening the constraint $\C$ is a natural option to attain convexity. However, we point out that such a change can sometimes raise other challenges, like computing the causal model needed for counterfactual fairness.

\subsection{Radically changing the models}\label{sec:model}

A second strategy for dealing with a convex constraint consists in changing the nature of the base models $f$, that is, changing the space of models $\G$ into a radically different one. In this subsection, we focus on replacing the deterministic mapping $f$ by a random coupling. We need some notation: $\operatorname{I}_d$ refers to the identity function on $\R^d$ and for any $f \in \G$, $(\operatorname{I}_d, f)$ denotes the function $ \R^d \to \R^d \times \R^p : x \mapsto (x,f(x))$.

An illustration of this strategy comes from optimal-transport theory. Let $p=d$ and consider the Monge formulation of optimal transport between $P$ and $Q$ in $\mathcal{P}(\R^d)$:
\[
\min_{f \in \T(P,Q)} \int \norm{x-f(x)}^2 \mathrm{d}P(x).
\]
Although the loss to minimize is actually convex, the constraint $\T(P,Q)$ renders the problem nonconvex in most configurations according to \cref{cor:transport}. Interestingly, one can recover convexity by rewriting $\T(P,Q)$ with random couplings $\pi$ rather than deterministic functions $f$. Let us denote by $\Pi(P,Q) \subset \mathcal{P}(\R^d \times \R^d)$ the set of couplings with $P$ and $Q$ as, respectively, the first and second marginals, which is a convex set. The Kantorovich  formulation of optimal transport \citep{kantorovich1958space} addresses the following relaxation:
\[
\min_{\pi \in \Pi(P,Q)} \int \norm{x-y}^2 \mathrm{d}\pi(x,y).
\]
Not only does this problem always admit a solution in contrast to Monge's version, but it also remains convex since $\Pi(P,Q)$ is a convex set for any $(P,Q)$ and $\pi \mapsto \int \norm{x-y}^2 \mathrm{d}\pi(x,y)$ is a convex function. All in all, the underlying general idea of Kantorovich's relaxation consists in replacing the \emph{deterministic} coupling $(\operatorname{I}_d, f)_\sharp P$ induced by $f$---which belongs to $\Pi(P,Q)$ if and only if $f_\sharp P = Q$---by \emph{any} coupling $\pi$ in $\Pi(P,Q)$. This substitution of $\T(P,Q)$ by $\Pi(P,Q)$ could be generalized to attain convexity at the model level (not necessarily the parameter level) in any learning problems with a transport-map constraint, even when $p \neq q$.

This raises the question whether such a substitution can similarly render the set $\mathcal{E}(P,Q)$ convex. In the context of equalizing maps, the coupling reformulation is not as natural as for transport maps. To illustrate this, let us rewrite $\mathcal{E}(P,Q)$ as
\[
    \mathcal{E}(P,Q) = \bigcup_{R \in \mathcal{P}(\R^p)} \T(P,R) \cap \T(Q,R).
\]
First, we emphasize that directly exchanging $\T$ by $\Pi$ in the above expression to obtain a coupling reformulation is inappropriate. It would produce an empty set as soon as $P \neq Q$, since in this case $\Pi(P,R) \cap \Pi(Q,R) = \emptyset$ for any $R \in \mathcal{P}(\R^p)$. A more applicable reformulation consists in replacing $\T(P,R) \cap \T(Q,R)$ by $\Pi(P,R) \times \Pi(Q,R)$, leading to the definition of
\[
    \Gamma(P,Q) := \bigcup_{R \in \mathcal{P}(\R^p)} \Pi(P,R) \times \Pi(Q,R).
\]
This generalization of $\mathcal{E}(P,Q)$ considers \emph{pairs} of couplings rather than \emph{single} couplings. It captures the fact that the constraint $f_\sharp P = f_\sharp Q$ involves \emph{two} couplings induced by $f$, namely $(\operatorname{I}_d \times f)_\sharp P$ and $(\operatorname{I}_d, f)_\sharp Q$, which share the same right marginal. Crucially, $\Gamma(P,Q)$ is a nonempty convex set.\footnote{To prove this point, one can simply check that the definition of convexity is satisfied. We recall that convexity is not stable under union in general.} In the fairness setting of \cref{sec:parity}, one could therefore learn a pair of random predictors $(\pi_P,\pi_Q) \in \Gamma(P,Q)$ rather than a single deterministic predictor $f \in \mathcal{E}(P,Q)$ to recover convexity. Then, one would apply $\pi_P$ or $\pi_Q$ (or rather their corresponding random mappings) on realizations of, respectively, $P$ and $Q$, to obtain their output distributions.

\begin{remark}[Not in the published version]
While this \say{randomization} of the deterministic couplings does convexify the sets $\T(P,Q)$ and $\mathcal{E}(P,Q)$ in theory, it may fail to render the associated learning problems convex in practice. To encode couplings into computational objects, people typically rely on \emph{noise outsourcing}, according to which all random couplings can be represented as random mappings. Notably, \cite{korotin2023neural} used this representation to train neural \emph{stochastic} optimal transport plans. The general approach consists in adding an independent source of randomness to the inputs of the model $f$. Let us illustrate in the case of transport maps. Rather than learning $f \in \T(P,Q)$, one designs $f \in \T(P \otimes \nu, Q)$ where $\nu$ is continuous. As such, $f(x,\cdot)_\sharp \nu$ describes the probability kernel conditional to $x$ of a coupling with $P$ and $Q$ as marginals. However, $\T(P \otimes \nu, Q)$ suffers from the same convexity issues as any set of transport maps. This raises the question of whether there exists practical convex representations of random couplings.
\end{remark}

\section{Conclusion}

By diving into the theory of transport maps and equalizing maps, we showed that pushing-forward measures was a nonconvex operation in general. Analyzing popular machine-learning problems in light of this underappreciated characteristic enabled us to provide a structural understanding of their (non)convexity. This will hopefully help practitioners and researchers know when it is in vain to try designing a convex objective and consequently encourage them to rapidly consider a different approach if convexity is required.

Our work also opens further lines of research. From a mathematical perspective, there is still much to discover regarding the shape of push-forward constraints, notably equalizing maps. From an applicative angle, a valuable direction would be to investigate deeper the ideas from \cref{sec:recover} to furnish better guidance on the construction of alternative convex problems.

\appendix

\section{Equalizing maps between finitely supported probability measures}\label{sec:stump} This supplementary section provides insight on $\mathcal{E}(P,Q)$, notably on its convexity, when $P$ and $Q$ are two discrete probability measures with finite supports.

\subsection{Setup}

Formally, let $n,m \geq 1$ be two integers, $\{\alpha_i\}^n_{i=1},\{\beta_j\}^n_{i=1} \subset [0,1]$ be two sets of probability weights, and $\{x_i\}_{i=1}^n, \{y_j\}_{j=1}^m \subset \R^d$ be two sets of values. Then, we define $P$ and $Q$ as $P :=\sum_{i=1}^n \alpha_i \delta_{x_i}$ and $Q :=\sum_{j=1}^m \beta_j \delta_{y_j}$. In a first time, we restrict our analysis to measures with disjoint supports, which means supposing that $\{x_i\}_{i=1}^n \cap \{y_j\}_{j=1}^m = \emptyset$. Later, we explain through \cref{prop:nondisjoint} how to extend our results to measures with nondisjoint supports.

Let us introduce extra notation before proceeding. We denote by $[k]$ the set $\{1,\cdots, k\}$ and by $2^{[k]}$ the set of parts of $[k]$ for any integer $k \geq 1$. We also define the following sets:
\begin{equation*}
    S_{\alpha}=\left\{\sum_{i\in I} \alpha_i,\, I \in 2^{[n]}\right\}, \hspace{1cm} S_{\beta}=\left\{\sum_{i\in J} \beta_j,\, J \in 2^{[m]}\right\}.  
\end{equation*}
They correspond to the reachable values by sums of the weights of, respectively, $P$ and $Q$, and can be seen as the image sets of these discrete measures. They will play a key role in the results below. Finally, we write $S_{\alpha,\beta} := S_\alpha \cap S_\beta$. Note that $\{0,1\} \subseteq S_\alpha, S_\beta \subset [0,1]$.

\subsection{Characterization of convexity}
The following theorem gives a sufficient and necessary condition on $P$ and $Q$ to have the convexity of the set $\mathcal{E}(P,Q)$. 
\begin{theorem}[equalizing maps between finitely supported discrete measures]\label{thm:discreteCharacConv}
Let $n,m \geq 1$ be two integers, $\{x_i\}^n_{i=1},\{y_j\}^m_{j=1} \subset \R^d$ be two sets of two-by-two distinct elements such that $\{x_i\}^n_{i=1} \cap \{y_j\}^m_{j=1} = \emptyset$, and $\{\alpha_i\}^n_{i=1},\{\beta_j\}^m_{j=1} \subset [0,1]$ be probability weights. Define $P :=\sum_{i=1}^n \alpha_i \delta_{x_i}$ and $Q :=\sum_{j=1}^m \beta_j \delta_{y_j}$. Then, the set $\mathcal{E}(P,Q)$ is convex if and only if the three following conditions hold:
\begin{itemize}
    \item[(i)] for every $\gamma$ in $S_{\alpha, \beta}$, there exists a unique couple $(I_\gamma, J_\gamma)\in 2^{[n]} \times 2^{[m]}$ such that 
    \[
    \sum_{i \in I_\gamma} \alpha_i =\sum_{j \in J_\gamma} \beta_j =\gamma;
    \]
    \item[(ii)] the sets $\{I_\gamma\}_{\gamma \in S_{\alpha,\beta}}$ and $\{J_\gamma\}_{\gamma \in S_{\alpha,\beta}}$ defined by (i) are $\sigma$-algebras of, respectively, $[n]$ and $[m]$;
    \item[(iii)] for every $\gamma,\gamma' \in S_{\alpha,\beta}$, the index sets $\eta_\alpha(\gamma,\gamma'), \eta_\beta(\gamma,\gamma') \in S_{\alpha,\beta}$ defined by (i) and (ii) such as $I_{\eta_\alpha(\gamma,\gamma')} = I_\gamma \cap I_{\gamma'}=$ and $J_{\eta_\beta(\gamma,\gamma')} = J_\gamma \cap J_{\gamma'}$ satisfy
    \[
    \eta_\alpha(\gamma,\gamma')=\eta_\beta(\gamma,\gamma').
    \]
\end{itemize}
\end{theorem}

The above conditions may seem convoluted. Nevertheless, note that item (i) alone is a strong condition on $P$ and $Q$. It requires that \emph{every} probability attainable by both $P$ and $Q$ corresponds to unique events for, respectively, $P$ and $Q$. Therefore, it provides a powerful criterion to identify settings where $\mathcal{E}(P,Q)$ is not convex. This is precisely the strategy we follow to prove \cref{prop:empirical_equalizing}.
\begin{proof}[\cref{prop:empirical_equalizing}] We address each item separately.

First, let us suppose that $m$ and $n$ are coprime. Notice that if $S_{\alpha,\beta} := S_\alpha \cap S_\beta =\{0,1\}$, then $\mathcal{E}(P,Q)$ is the set of constant functions over $\{x_i\}_{i=1}^n \cup \{y_j\}_{j=1}^m$. Therefore, we aim at showing in this part of the proof that $S_{\alpha, \beta} =\{0,1\}$. By the definitions of $P$ and $Q$ we have 
\[S_\alpha=\{k/n,\, 0\leq k\leq n\}, \hspace{1cm} S_\beta=\{k'/m,\, 0\leq k'\leq m\}.\]
Consequently, for any $\gamma \in S_{\alpha,\beta}$, there exist two integers $0\leq k\leq n$ and $0 \leq k'\leq m$ such that $\gamma= k/n=k'/m$, and hence $k'n= km$. This entails that $n$ divides $km$, and thereby $n$ divides $k$ since $n$ and $m$ are coprime. Finally, recall that $0\leq k\leq n$, leading to $k=n$ or $k=0$, and therefore $\gamma=1$ or $\gamma=0$.

Second, let us suppose that $m$ and $n$ are not coprime: there exist three integers $r \geq 2$ and $n',m' \geq 1$ such that $n=rn'$ and $m=rm'$. For every element $I\in 2^{[n]}$ of size $n'$ and element $J\in 2^{[m]}$ of size $m'$ we have 
\[
\frac{1}{r}=\sum_{i\in I} \alpha_i=\sum_{j\in J} \beta_j.
\]
This contradicts condition (i) of \cref{thm:discreteCharacConv}, as the number of elements $I\in 2^{[n]}$ of size $n'$ is larger than one. Hence, $\mathcal{E}(P,Q)$ is not convex.
\end{proof}

The proof of \cref{thm:discreteCharacConv} is divided in two parts, one for each side of the equivalence. We address the necessary condition by contrapositive; proving the sufficient condition is more straightforward.

\begin{proof}[\cref{thm:discreteCharacConv}]
Let us start with the \textbf{necessary condition ($\Longleftarrow$)}. First, suppose that condition (i) does not hold: there exists $\gamma\in S_{\alpha, \beta}$ such that there are two distinct couples $(I_\gamma,J_\gamma),(I'_\gamma, J'_\gamma) \in 2^{[n]} \times 2^{[m]}$ verifying $\sum_{i\in I_\gamma} \alpha_i=\sum_{j \in J_\gamma} \beta_i= \sum_{i\in I_{\gamma'}} \alpha_i=\sum_{j \in J_{\gamma'}} \beta_i = \gamma$. Without loss of generality, we can assume that $I_\gamma \neq I'_\gamma$, meaning that there exists $i_0 \in I_\gamma$ such that $i_0 \notin I'_\gamma$. Then, for $z_1$ and $z_2$ two distinct elements of $\R^p$, we define the functions $f,g \in \G$ as
\[
f(x)=
    \begin{cases}
    z_1 & \mbox{if }  x \in \{x_i,\, i \in I_\gamma\}\cup \{y_j,\, j \in  J_\gamma\}\\
    z_2 & \mbox{otherwise } 
    \end{cases}
\]
and
\[
g(x)=
    \begin{cases}
    z_1 & \mbox{if }  x \in \{x_i,\, i \in I'_\gamma\}\cup \{y_j,\, j \in  J_\gamma\},\\
    z_2  & \mbox{otherwise }
    \end{cases}.
\]
Since $\gamma \in S_{\alpha,\beta}$, these functions satisfy $
f_\sharp P = f_\sharp Q = g_\sharp P = g_\sharp Q = \gamma \delta_{z_1}+ (1-\gamma)\delta_{z_2}.$ In addition,
\[
\left(\frac{1}{2}f+\frac{1}{2}g\right)_\sharp P = \alpha_{i_0}\delta_{\frac{z_1+z_2}{2}} + \sum_{i \in [n], i\neq i_0 }\alpha_i \delta_{\frac{f(x_i)+g(x_i)}{2}} \neq
\left(\frac{1}{2}f+\frac{1}{2}g\right)_\sharp Q = \gamma\delta_{z_1}+ (1-\gamma)\delta_{z_2}.
\]
This proves that $\mathcal{E}(P,Q)$ is not convex.

Second, suppose that (i) holds but (ii) does not hold. It readily follows from $\sum_{i =1}^n \alpha_i = \sum_{j=1}^m \beta_j = 1$  that for every $\gamma \in S_{\alpha,\beta}$, $I^c_\gamma := [n] \setminus I_\gamma=I_{1-\gamma}$ and $J^c_\gamma := [m] \setminus J_\gamma = J_{1-\gamma}$. This implies that both $\{I_\gamma\}_{\gamma \in S_{\alpha,\beta}}$ and $\{J_\gamma\}_{\gamma \in S_{\alpha,\beta}}$ are stable under complementation and not empty. Therefore, (ii) being false means that $\{I_\gamma\}_{\gamma \in S_{\alpha,\beta}}$ or $\{J_\gamma\}_{\gamma \in S_{\alpha,\beta}}$ is not stable under intersection. Let us find a counterexample for the convexity of $\mathcal{E}(P,Q)$ using this property. Without loss of generality, we can assume that $\{I_\gamma\}_{\gamma \in S_{\alpha, \beta}}$ is not stable under intersection: there exist $\gamma_1,\gamma_2 \in S_{\alpha,\beta}$ such that $I_{\gamma_1}\cap I_{\gamma_2} \notin \{I_\gamma\}_{\gamma \in S_{\alpha, \beta}}$. Then, for $z_1$ and $z_2$ two distinct elements of $\R^p$, we define the functions $f,g \in \G$ as
\[
f(x)=
    \begin{cases}
    z_1 & \mbox{if }  x \in \{x_i,\, i \in I_{\gamma_1}\}\cup \{y_j,\, j \in  J_{\gamma_1}\}\\
    z_2 & \mbox{otherwise } 
    \end{cases}
\]
and
\[
g(x)=
    \begin{cases}
    z_1 & \mbox{if }  x \in \{x_i,\, i \in I_{\gamma_2}\}\cup \{y_j,\, j \in  J_{\gamma_2}\}\\
    z_2  & \mbox{otherwise }
    \end{cases}.
\]
As before, it follows from $\gamma_1,\gamma_2 \in S_{\alpha,\beta}$ that $f_\sharp P= f_\sharp Q= \gamma_1 \delta_{z_1} + (1-\gamma_1) \delta_{z_2}$ and $g_\sharp P= g_\sharp Q = \gamma_2 \delta_{z_1} + (1-\gamma_2) \delta_{z_2}$. In addition,
\begin{multline*}
\left(\frac{1}{2}f+\frac{1}{2}g\right)_\sharp P =
\left(\sum_{i\in I_{\gamma_1}\cap I_{\gamma_2}} \alpha_i\right) \delta_{z_1}\\+ \left(\sum_{i\in I^c_{\gamma_1} \cap I^c_{\gamma_2}} \alpha_i \right) \delta_{z_2} + \left(\sum_{i\in \left(I_{\gamma_1}\cup I_{\gamma_2} \right) \setminus \left( I_{\gamma_1}\cap I_{\gamma_2} \right)} \alpha_i \right)\delta_{\frac{z_1+z_2}{2}}
\end{multline*}
and
\begin{multline*}
\left(\frac{1}{2}f+\frac{1}{2}g\right)_\sharp Q =
\left(\sum_{j \in J_{\gamma_1}\cap J_{\gamma_2}} \beta_j\right) \delta_{z_1}\\+ \left(\sum_{j\in J^c_{\gamma_1} \cap J^c_{\gamma_2}} \beta_j \right) \delta_{z_2} + \left(\sum_{j \in \left( J_{\gamma_1}\cup J_{\gamma_2}\right)\setminus \left(J_{\gamma_1} \cap J_{\gamma_2} \right)} \beta_j \right)\delta_{\frac{z_1+z_2}{2}}.
\end{multline*}
Critically, $I_{\gamma_1} \cap I_{\gamma_2} \notin \{I_\gamma\}_{\gamma \in S_{\alpha, \beta}}$ entails that $\sum_{i\in I_{\gamma_1}\cap I_{\gamma_2}} \alpha_i \neq \sum_{j\in J_{\gamma_1}\cap J_{\gamma_2}} \beta_j$. Therefore, $\left(\frac{1}{2}f+\frac{1}{2}g\right)_\sharp P \neq \left(\frac{1}{2}f+\frac{1}{2}g\right)_\sharp Q$, meaning that $\mathcal{E}(P,Q)$ is not convex. 

Finally, suppose that (i) and (ii) hold and that (iii) does not hold: there exist $\gamma_1,\gamma_2 \in S_{\alpha,\beta}$ such that $\eta_\alpha(\gamma_1,\gamma_2)\neq \eta_\beta(\gamma_1,\gamma_2)$. Let us define $f,g \in \G$ as before. They still verify $f_\sharp P = f_\sharp Q$ and $g_\sharp P = g_\sharp Q$. In addition,
\begin{multline*}
\left(\frac{1}{2}f+\frac{1}{2}g\right)_\sharp P =
\eta_\alpha(\gamma_1,\gamma_2) \delta_{z_1}\\+ \eta_\alpha(1-\gamma_1,1-\gamma_2) \delta_{z_2} + \left( \eta_\alpha(\gamma_1,1-\gamma_2) + \eta_\alpha(1-\gamma_1,\gamma_2) \right) \delta_{\frac{z_1+z_2}{2}},    
\end{multline*}
and
\begin{multline*}
\left(\frac{1}{2}f+\frac{1}{2}g\right)_\sharp Q =
\eta_\beta(\gamma_1,\gamma_2) \delta_{z_1}\\+ \eta_\beta(1-\gamma_1,1-\gamma_2) \delta_{z_2} + \left(\eta_\beta(\gamma_1,1-\gamma_2) + \eta_\beta(1-\gamma_1,\gamma_2)\right) \delta_{\frac{z_1+z_2}{2}}.
\end{multline*}
By assumption, $\eta_\alpha(\gamma_1,\gamma_2)\neq \eta_\beta(\gamma_1,\gamma_2)$. Hence $\left(\frac{1}{2}f+\frac{1}{2}g\right)_\sharp P \neq \left(\frac{1}{2}f+\frac{1}{2}g\right)_\sharp Q$, meaning that $\mathcal{E}(P,Q)$ is not convex.

We now turn to the \textbf{sufficient condition ($\Longrightarrow$)}. Consider that conditions (i), (ii), and (iii) hold. By definition, for any $f,g \in \mathcal{E}(P,Q)$ there exist two integers $K,L \geq 1$, two sets of probability weights $\{\gamma_k\}^K_{k=1},\{\gamma'_l\}^L_{l=1}$, and two sets of values $\{f_k\}^K_{k=1},\{g_l\}^L_{l=1} \subseteq \R^p$ such that 
\[
f_\sharp P = f_\sharp Q = \sum_{k=1}^K \gamma_k \delta_{f_k}  
\ \text{and}\
g_\sharp P = g_\sharp Q = \sum_{l=1}^L \gamma'_l \delta_{g_l}.
\]
Then, according to (i) and (ii)
\[
[n]=\bigcup_{k=1}^K I_{\gamma_k} = \bigcup_{l=1}^L I_{\gamma'_l}= \bigcup_{k=1}^K \bigcup_{l=1}^L I_{\eta_\alpha(\gamma_k,\gamma'_l)},
\]
where $I_{\eta_\alpha(\gamma_k,\gamma'_l)}=I_{\gamma_k}\cap I_{\gamma'_l} $. Condition (i) gives $I_{\gamma_k}=\{i \in [n], f(x_i)=f_k\}$ and $I_{\gamma'_l}=\{i \in [n], f(x_i)=g_l\}$. Similarly,
\[
[m]=\bigcup_{k=1}^K J_{\gamma_k} = \bigcup_{l=1}^L J_{\gamma'_l}= \bigcup_{k=1}^K\bigcup_{l=1}^L J_{\eta_\beta(\gamma_k,\gamma'_l)},
\]
with $J_{\gamma_k}=\{j \in [m], f(y_j)=f_k\}$ and $J_{\gamma'_l}=\{j \in [m], f(y_j)=g_l\}$ again from condition (i). Then, for all $t \in [0,1]$ 
\begin{multline*}
\left(t f+(1-t)g\right)_\sharp P = \sum_{k=1}^K\sum_{l=1}^L \eta_\alpha(\gamma_k,\gamma'_l)\delta_{tf_k+(1-t)g_l}= \sum_{k=1}^K\sum_{l=1}^L \eta_\beta(\gamma_k,\gamma'_l)\delta_{tf_k+(1-t)g_l}\\= \left(t f+(1-t)g\right)_\sharp Q,
\end{multline*}
due to condition (iii). This conclude the proof.

\end{proof}

\subsection{When the probability measures have nondisjoint supports}

\Cref{thm:discreteCharacConv}, and thereby \cref{prop:empirical_equalizing}, hold for $P$ and $Q$ with disjoint supports. Nevertheless, one can deduce conditions for the convexity of $\mathcal{E}(P,Q)$ in more general cases by subtracting the common mass of $P$ and $Q$. More formally, for $P:=\sum_{i=1}^n \alpha_i\delta_{x_i}$ and $Q:=\sum_{j=1}^m\beta_j\delta_{y_j}$ two finitely supported elements of $\mathcal{P}(\R^d)$, we define their minimum as
\[
\min(P,Q) := \sum^n_{i=1} \sum^m_{j=1} \min(\alpha_i,\beta_j) \mathbf{1}_{\{x_i = y_j\}} \delta_{x_i},
\]
which is a finitely supported element of $\M(\R^d)$. Note that if $P$ and $Q$ have disjoint supports, then $\min(P,Q)$ is the null measure. Crucially, the following result hold.

\begin{proposition}[equalizing maps between finitely supported measures with nondisjoint supports]\label{prop:nondisjoint}
Let $n,m \geq 1$ be two integers, $\{x_i\}^n_{i=1},\{y_j\}^m_{j=1} \subset \R^d$ be two sets of two-by-two distinct elements, and $\{\alpha_i\}^n_{i=1},\{\beta_j\}^m_{j=1} \subset [0,1]$ be probability weights. Define $P:=\sum_{i=1}^n \alpha_i\delta_{x_i}$ and $Q:=\sum_{j=1}^m\beta_j\delta_{y_j}$. Then, $P-\min(P,Q)$ and $Q-\min(P,Q)$ have disjoint supports, share the same total mass, and satisfy
\[
\mathcal{E}(P,Q)=\mathcal{E}(P-\min(P,Q),Q-\min(P,Q)).
\]
\end{proposition}
This proposition enables one to determine conditions for the convexity of $\mathcal{E}(P,Q)$ by applying, for instance, the previous results to $P - \min(P,Q)$ and $Q - \min(P,Q)$.
\begin{proof}[\cref{prop:nondisjoint}]
The proof is trivial when $P$ and $Q$ already have disjoint supports. Let us assume that the supports of $P$ and $Q$ are not disjoint. Then, one can define a nonempty set $\{z_k\}_{k=1}^l := \{x_i\}_{i=1}^n \cap \{y_j\}_{j=1}^m$, where $l \geq 1$ is an integer. Without loss of generality, we reorder the elements in $\{x_i\}_{i=1}^n$ and $\{y_j\}_{j=1}^m$ so that $z_k=x_k=y_k$ for every $k \in [l]$. This enables us to express $\min(P,Q)$ as
\[
\min(P,Q) = \sum_{k=1}^l \min(\alpha_k,\beta_k) \delta_{z_k}.
\]
In the rest of the proof, we write $P' := P - \min(P,Q)$ and $Q' = Q - \min(P,Q)$.

First, we show that $P'$ and $Q'$ have disjoint supports. Note that 
\[
P'=\sum_{k=1}^l (\alpha_{k}-\min(\alpha_k,\beta_k))\delta_{z_k} + \sum_{i=l+1}^n \alpha_i\delta_{x_i} \ \text{and}\
Q'=\sum_{k=1}^l (\beta_{k}-\min(\alpha_k,\beta_k))\delta_{z_k} + \sum_{j=l+1}^m \beta_j\delta_{y_j}.
\]
By hypothesis, $\{x_i\}_{i=l+1}^n\cap \{y_j\}_{j=1}^m=\emptyset$. Additionally, for every $k \in [l]$, if $\alpha_k-\min(\alpha_k,\beta_k) > 0$, then $\alpha_k>\beta_k$, and hence $\beta_k-\min(\alpha_k,\beta_k)=0$. These two remarks ensure that $P'$ and $Q'$ have disjoint supports. They also clearly share the same total mass since $P$ and $Q$ both sum to one, and the subtrahend is the same, namely, $\min(P,Q)$.

Second, we prove that $\mathcal{E}(P',Q') = \mathcal{E}(P,Q)$. Every function $f$ in $\mathcal{E}(P',Q')$ satisfies $f_\sharp P' = f_\sharp Q'$ by definition. Therefore, it follows from $P = P' + \min(P,Q)$, $Q = Q' + \min(P,Q)$ and the linearity of the push-forward operation that
\[
f_\sharp P = f_\sharp P'+f_\sharp \min(P,Q)  = f_\sharp Q'+ f_\sharp \min(P,Q) = f_\sharp Q.
\]
This means that $\mathcal{E}(P',Q') \subseteq \mathcal{E}(P,Q)$. Conversely, for every $f$ in $\mathcal{E}(P,Q)$, the linearity of the push-forward yields
\[
f_\sharp P' = f_\sharp P - f_\sharp \min(P,Q) = f_\sharp Q - f_\sharp \min(P,Q) = f_\sharp \min(P,Q).
\]
Consequently, $\mathcal{E}(P,Q)\subseteq\mathcal{E}(P',Q')$. This concludes the proof.
\end{proof}

\begin{remark}[General measures with nondisjoint supports] \Cref{prop:equalizer,prop:nondisjoint} handle the case of equalizing maps between measures with nondisjoint supports in, respectively, the Lebesgue absolutely continuous case and the finitely supported case. They rely on the same idea: subtracting the common mass $\min(P,Q)$ to obtain $P':= P - \min(P,Q)$ and $Q' = Q - \min(P,Q)$ with disjoint supports, then noting that $\mathcal{E}(P,Q)=\mathcal{E}(P',Q')$. In \cref{prop:equalizer}, the minimum is defined via a truncated density function; in \cref{prop:nondisjoint}, it is defined via truncated probability weights. Interestingly, such a notion of minimum between measures can be defined for any probability measures using the Jordan decomposition \citep[Corollary 2.9]{kallenberg2021foundations}, which enables one to use the same subtraction argument in the general case.
\end{remark}

\section*{Acknowledgments}
We would like to thank Armand Foucault, Iain Henderson, and Graeme Baker for helpful discussions. We also thank the anonymous reviewers for their valuable feedback and helpful suggestions.

\bibliographystyle{abbrvnat}
\bibliography{references}

\end{document}